\documentclass{article} 
\usepackage{iclr2027_conference,times}

\usepackage{amsmath,amsfonts,bm}

\def\eqref#1{equation~\ref{#1}}

\def\1{\bm{1}}

\DeclareMathAlphabet{\mathsfit}{\encodingdefault}{\sfdefault}{m}{sl}
\SetMathAlphabet{\mathsfit}{bold}{\encodingdefault}{\sfdefault}{bx}{n}

\usepackage{hyperref}
\usepackage{url}
\usepackage{graphicx}
\usepackage{booktabs}
\usepackage{amsmath,amssymb,amsfonts}
\usepackage{amsthm}
\usepackage{float}

\theoremstyle{plain}
\newtheorem{proposition}{Proposition}
\newtheorem{corollary}{Corollary}[proposition] 

\theoremstyle{definition}

\theoremstyle{remark}
\newtheorem*{remark}{Remark}

\title{Understanding and Exploiting Anisotropy in Post-Training}

\author{%
  Samyak Jha\thanks{Equal contribution.} \\
  Indian Institute of Technology Dhanbad\\
  samyakjha71@gmail.com\\
  \And
  Harshvardhan Saini\footnotemark[1] \\
  Indian Institute of Technology Dhanbad\\
  hs1062005@gmail.com\\
  \And
  Yizhen Liao \\
  National University of Singapore\\
  liaoyz0711@gmail.com\\
  \And
  Yiming Tang\thanks{Corresponding authors.} \\
  National University of Singapore\\
  yiming@nus.edu.sg \\
  \And
  Dianbo Liu\footnotemark[2] \\
  National University of Singapore\\
  dianbo@nus.edu.sg \\
}

\newcommand{\method}{SphereGate}

\iclrfinalcopy
\begin{document}

\maketitle
\begin{abstract}
LLM post-training combines supervised fine-tuning (SFT), a mode-covering forward-KL objective, with reinforcement learning (RL), a mode-seeking reverse-KL objective. Frequency-weighted likelihood training leaves a well-known signature: \emph{anisotropy}, in which a few residual channels carry disproportionately large activations. Anisotropy is widely documented and usually treated as a defect, yet its function and its interaction with post-training remain unclear. We first analyze it. A label-free outlier rule isolates about 5\% of residual channels that are essential for language modeling: removing them raises perplexity from 10 to over $10^6$, versus 35 for count-matched random channels. Yet they barely distinguish correct from incorrect reasoning. SFT reshapes them, whereas RL leaves them largely intact and adapts the complementary channels. These channels therefore form the model's \emph{coherence substrate}, and reasoning adaptation happens elsewhere. We then exploit this. \textsc{SphereGate} learns one bounded gain per residual channel on a frozen backbone. Its activation-weighted gradients provably limit movement of high-energy coherence channels and leave the remaining channels free. With 0.1M trainable parameters, \textsc{SphereGate} outperforms parameter-efficient baselines by 2.0--7.3 points on MATH-500 across Qwen2.5 (0.5B--7B) and Llama-3-8B, is comparable or exceeds full-model GRPO. Anisotropy is not a defect but a division of labor that post-training can exploit.

\end{abstract}
\section{Introduction}
\label{sec:sg-v5-introduction}

Supervised fine-tuning (SFT) and reinforcement learning (RL) are central
to post-training large language models for reasoning
\citep{ouyang2022training,shao2024deepseekmath}. SFT learns from
demonstrations, while reward-driven RL increases the probability of
high-reward completions. Both operate on models with substantial
linguistic and reasoning structure, making the organization of their
existing representations a natural starting point for efficient
adaptation. A characteristic feature of these representations is
\emph{anisotropy}: token representations exhibit preferential
concentration along particular directions
\citep{gao2019representation,godey2024anisotropy}.
Figure~\ref{fig:anisotropy-overview}(a) illustrates this geometry.
Understanding its role during post-training offers a representation-level
perspective on how reasoning capabilities can be adapted.

\begin{figure}[t]
    \centering
    \includegraphics[width=0.8\linewidth]{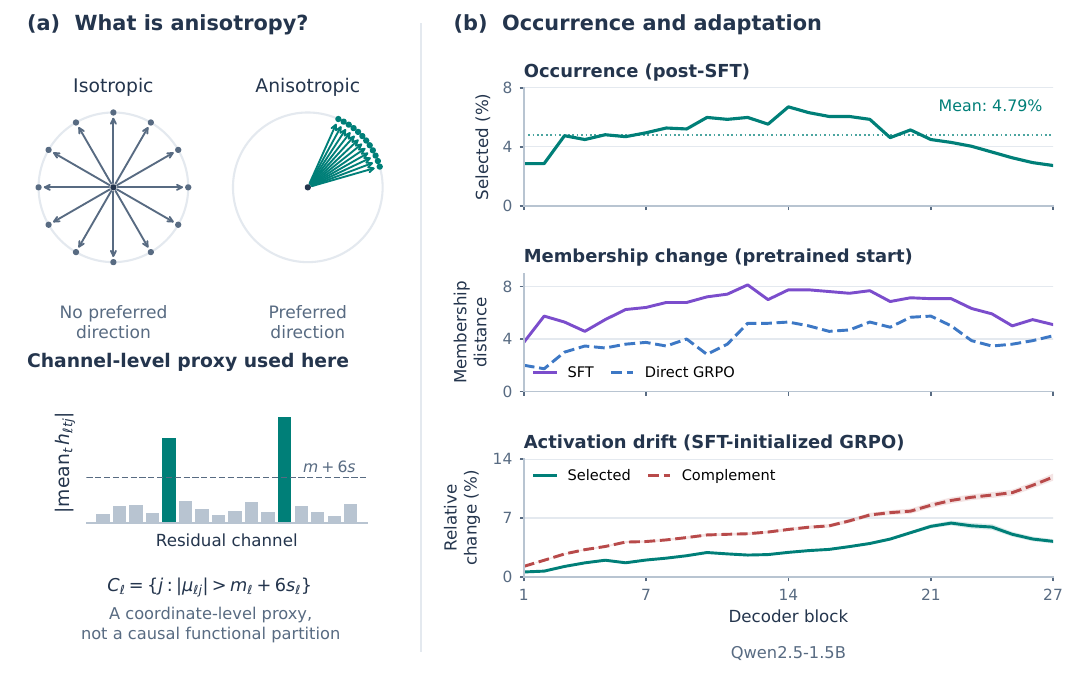}
    \caption{\textbf{Anisotropy and changes during post-training.}
    \textbf{(a)} Illustration of representations spread evenly across
    directions (isotropic) or concentrated around a preferred direction
    (anisotropic). The channel sketch highlights components whose
    average activations are unusually far from zero.
    \textbf{(b)} \emph{Top:} these selected channels form a small group
    across the displayed layers. \emph{Middle:} starting from the same
    pretrained model, supervised fine-tuning (SFT) changes which channels
    belong to this group more than direct reward-based training with
    GRPO in the compared runs. \emph{Bottom:} in a separate GRPO run
    starting from SFT, the remaining channels change more relative to
    their starting activation levels. Shading shows uncertainty estimated
    by repeatedly resampling evaluation problems.}
    \label{fig:anisotropy-overview}
\end{figure}

Prior studies document anisotropy and frequency-associated structure in
language-model representations
\citep{martinez2024mitigating,godey2024anisotropy,gao2019representation},
while circuit-level analyses identify attention mechanisms and neurons
that contribute to reasoning
\citep{arora2026language,kim2025reasoning}. Together, these findings
establish representation structure as a useful lens for studying model
behavior and motivate a closer examination of residual channels during
post-training. We focus on two questions: \emph{how do selected
anisotropic channels and their complement evolve under SFT and RL,
and how effectively can channel-wise reweighting adapt reasoning
within a frozen backbone?} These questions connect the geometry of
learned representations to the design of compact adaptation spaces.

In this study, we first characterize the occurrence and post-training
dynamics of anisotropic residual channels. A label-free outlier rule
on absolute mean post-block activations identifies a selected group
that we use as a coherence proxy. In our Qwen2.5-1.5B diagnostics, this
group comprises approximately 5\% of post-SFT coordinates across the
blocks displayed in Figure~\ref{fig:anisotropy-overview}(b). Comparing
independent training branches from the same pretrained checkpoint,
we observe larger selected-channel membership changes after SFT than
after direct GRPO. A separate SFT-initialized full-model GRPO run
exhibits larger baseline-relative activation changes in the
complementary channels. Our rollout analysis also reports higher
normalized subspace overlap between correct and incorrect rollouts
in the selected group than in its complement
(Figure~\ref{fig:analysis}). These observations
provide a channel-level account of heterogeneous adaptation dynamics
and motivate learning how to reweight the existing residual structure.

Building on this characterization, we introduce \textbf{SphereGate}
to manipulate the relative weighting of residual channels during
reasoning-oriented RL. Following the representation-level perspective
of activation-scaling methods such as IA\textsuperscript{3}
\citep{liu2022few}, SphereGate learns one positive gain per channel
on the full residual state after each decoder block, with all backbone
weights held fixed. The gains are identity-initialized and have bounded
log-values, yielding a compact diagonal adaptation operator. All
channel gains are trainable in the default configuration, while the
diagnostic partition supports analysis and restricted-gate ablations.
Across four backbones on MATH-500, SphereGate exceeds the strongest
TinyLoRA, LoRA-XS, or subnetwork baseline by
\textbf{2.0--7.3 percentage points}. Complement-only adaptation retains
nearly the all-channel gate's accuracy, further supporting the utility
of channel-wise residual reweighting for reasoning.

Our contributions are threefold:
\begin{itemize}
    \item \textbf{Characterizing anisotropic-channel dynamics.}
    We define a label-free mean-activation outlier partition and
    analyze its occurrence, rollout-subspace overlap, membership
    changes, and baseline-relative activation drift under SFT and GRPO.

    \item \textbf{Manipulating residual representations with a frozen backbone.}
    We introduce SphereGate, an identity-initialized positive
    full-residual gating method, and characterize its stability through
    an energy-weighted analysis under stated conditions.

    \item \textbf{Evaluating adaptation across scales and architectures.}
    Experiments on Qwen2.5-Instruct at 0.5B--7B and
    Llama-3-8B-Instruct demonstrate mathematical reasoning gains over
    the evaluated PEFT baselines, complemented by channel-restricted
    ablations.
\end{itemize}

\section{Related Work}

\paragraph{Parameter-efficient reinforcement learning.}
Recent work has shown that reinforcement learning for reasoning can be performed effectively by updating only a small subset of model parameters. TinyLoRA \citep{morris2026learning}, LoRA-XS \citep{balazy2024lora}, and the subnetwork view of RL \citep{mukherjee2025reinforcement} demonstrate that substantial reasoning gains can be obtained with highly parameter-efficient updates. These approaches, however, primarily identify the trainable parameter subspace using algebraic or parameter-space criteria, without explicitly considering the representational geometry of the model or the functional roles of individual representation channels \citep{dai2025san,liu2024dora}. Classical parameter-efficient fine-tuning methods such as LoRA \citep{hu2021lora} operate in weight space. Our experiments use Group Relative Policy Optimization (GRPO) \citep{shao2024deepseekmath}, a reinforcement-learning objective designed for reasoning with verifiable rewards. In contrast to weight-space approaches, SphereGate exploits the representation-space structure identified in our analysis and performs channel-wise gating of hidden representations, directing the adaptation budget toward reasoning-relevant channels while minimizing disruption to coherence-related channels.

\paragraph{Mechanistic interpretability and representation editing.}
A growing body of work has sought to identify the internal mechanisms responsible for reasoning in language models \citep{olsson2022context,du2026does,tang2026capsule,zhao2511rep2text}. Recent studies analyze attention-based circuits and individual neurons to identify components that contribute causally to reasoning behavior \citep{arora2026language,kim2025reasoning}, providing evidence that reasoning is mediated by structured and often sparse subsets of model components. A parallel line of work studies how model behavior can be controlled through interventions on internal representations \citep{zou2023representation}, including steering via activation differences \citep{turner2024steering,panickssery2023steering}, learned representation transformations such as ReFT and LoReFT \citep{wu2404reft}. These analyses, however, primarily operate at the level of attention heads, neurons, or architectural components rather than the residual channel structure. Our work instead investigates the residual representation at the channel level, identifying a distinction between reasoning-oriented channels and a separate set of highly anisotropic channels associated with coherent behavior, and directly exploits this structure during RL.

\section{Analysis}
\label{sec:analysis}

\begin{figure*}[t]
\centering
\begin{minipage}[t]{0.5\textwidth}
\centering
\includegraphics[width=0.7\linewidth]{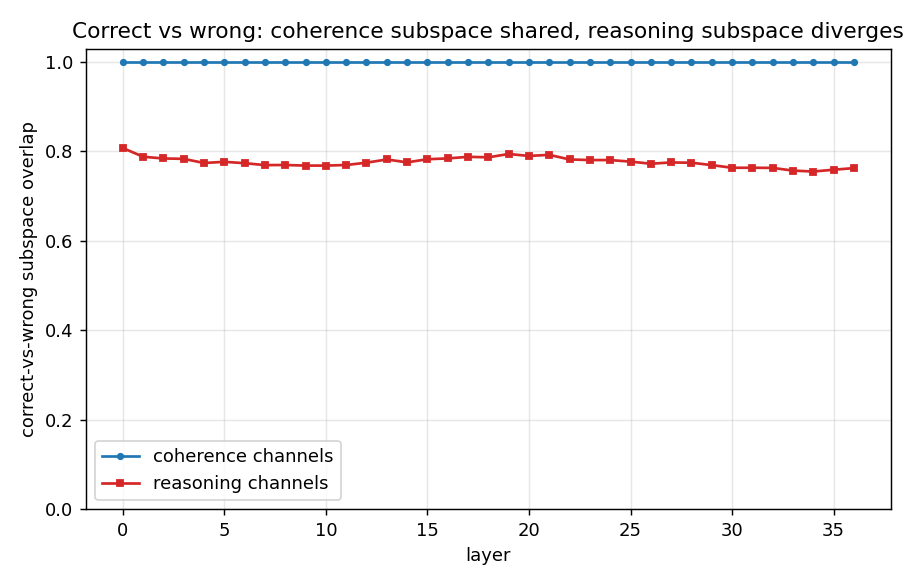}
\par\smallskip
\textbf{(a)}
\end{minipage}\hfill%
\begin{minipage}[t]{0.5\textwidth}
\centering
\includegraphics[width=0.7\linewidth]{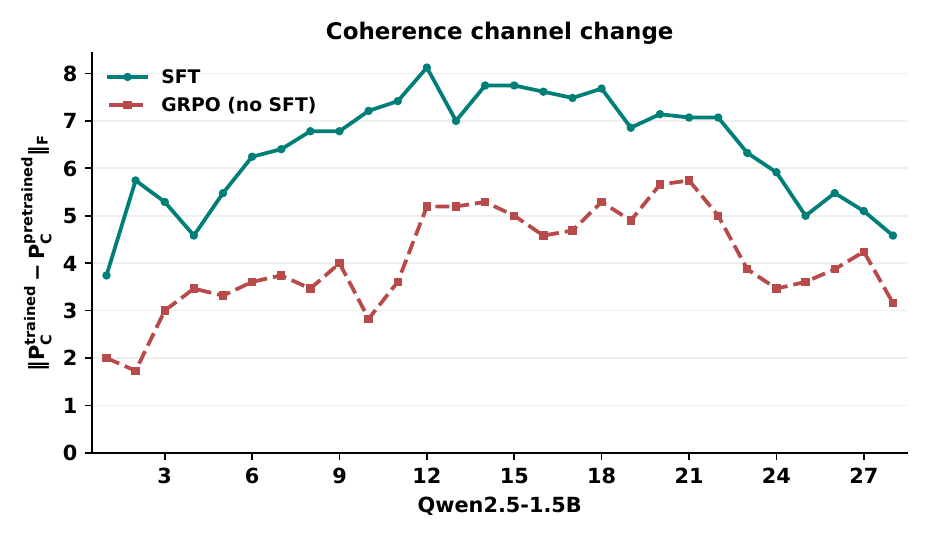}
\par\smallskip
\textbf{(b)}
\end{minipage}
\par\medskip
\begin{minipage}[t]{0.5\textwidth}
\centering
\includegraphics[width=0.7\linewidth]{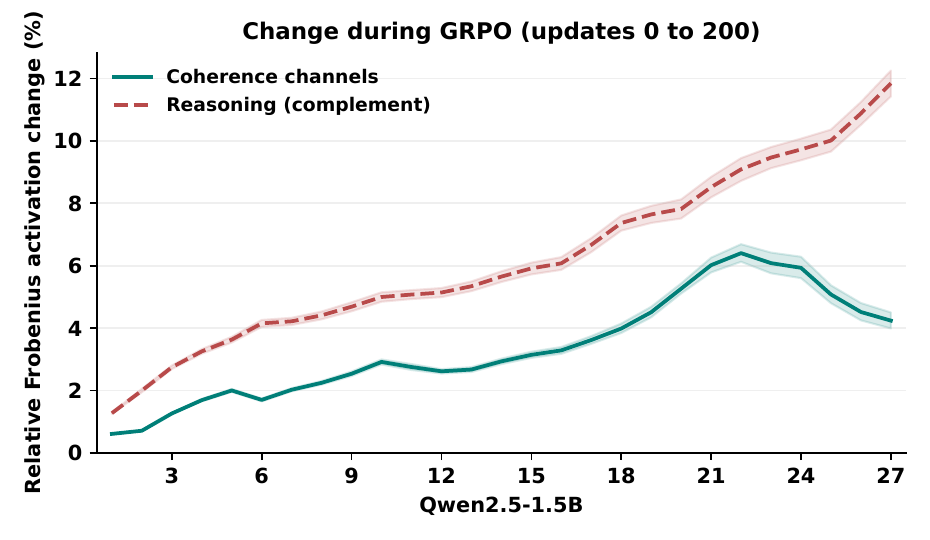}
\par\smallskip
\textbf{(c)}
\end{minipage}\hfill%
\begin{minipage}[t]{0.5\textwidth}
\centering
\includegraphics[width=0.7\linewidth]{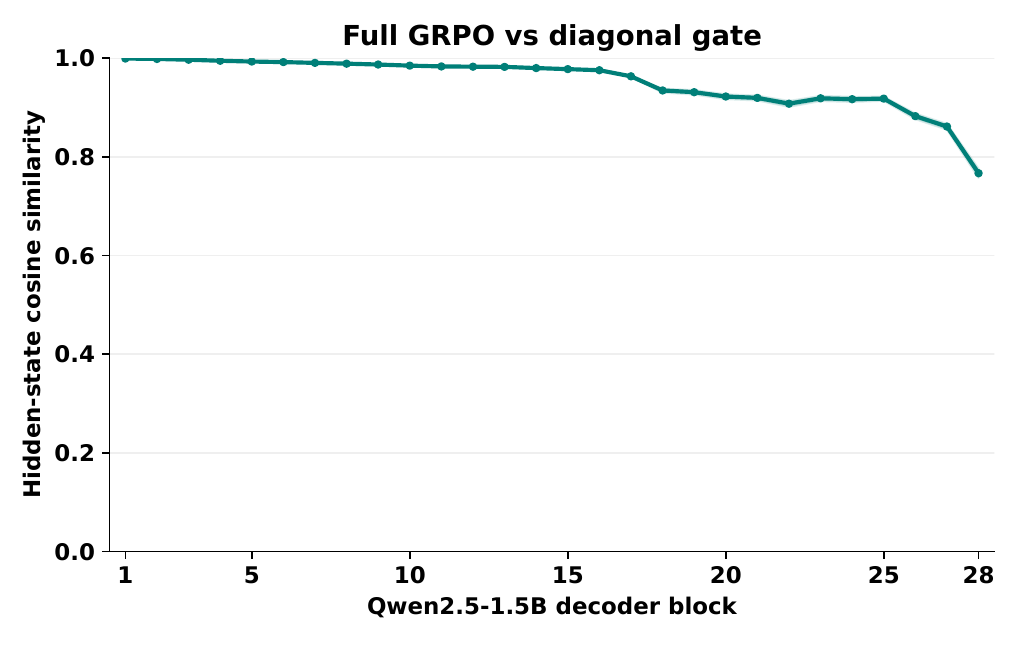}
\par\smallskip
\textbf{(d)}
\end{minipage}
\caption{\textbf{Coherence and reasoning representations.}
\textbf{(a)} Correct and incorrect rollouts have higher subspace overlap in coherence channels than in their complement.
\textbf{(b)} SFT produces larger coherence-channel membership changes than direct GRPO relative to the same pretrained model.
\textbf{(c)} During SFT-initialized GRPO, complementary reasoning channels undergo larger baseline-relative activation changes than coherence channels.
\textbf{(d)} Hidden-state cosine similarity between non-anisotropic diagonal-gate GRPO and full-model GRPO, both initialized from the same pretrained checkpoint, across decoder blocks~1--28.
Shading denotes pointwise 95\% problem-bootstrap confidence intervals.}
\label{fig:analysis}
\end{figure*}

We examine how correct and incorrect trajectories align and how SFT and GRPO modify their representations. We use anisotropic channels as a coherence proxy and refer to their complement as reasoning channels. For a matrix $A$, the Frobenius norm is
\begin{equation}
\|A\|_F=\left(\sum_{i,j}|A_{ij}|^2\right)^{1/2}.
\label{eq:frobenius-norm}
\end{equation}

\subsection{Diagonal-Gate versus Full-GRPO Similarity.}
In a separate matched-budget experiment, we compare non-anisotropic diagonal-gate GRPO with full-model GRPO, both trained independently from the same pretrained Qwen2.5-1.5B checkpoint without SFT initialization. Each run uses 1,000 rollout batches and 64,000 trajectories. Figure~\ref{fig:analysis}(d) reports the mean per-token cosine similarity of their post-block residual states on identical teacher-forced prefixes from 128 held-out gold solutions, pooled over 27,742 prediction positions at each decoder block. States are compared in their native coordinates without fitted alignment; the final block is measured before the final RMSNorm. Similarity is high in early blocks and decreases slightly motivating the existence of a diagonal only type of gate for reinforcement learning algorithms as post training for Large Language models.

\subsection{Correct and Wrong Subspace Alignment}
\label{sec:subspace_alignment}
We compare the representation subspaces of correct and incorrect rollouts, using normalized Frobenius overlap within each channel group. We generate correct and incorrect rollouts for each prompt and then independently aggregate them to construct the correct and incorrect subspaces with spectral decomposition, respectively. For group $g\in\{C,R\}$, let $U_g^{+}$ and $U_g^{-}$ be orthonormal bases for the correct and incorrect subspaces, with nonzero ranks $r_g^{+}$ and $r_g^{-}$. Their overlap is
\begin{equation}
\mathcal{O}_g=\frac{\left\|(U_g^{+})^{\top}U_g^{-}\right\|_F}{\sqrt{\min(r_g^{+},r_g^{-})}}.
\label{eq:correct-wrong-overlap}
\end{equation}
Figure~\ref{fig:analysis}(a) shows nearly complete alignment in the coherence channels, whereas the complementary reasoning channels exhibit lower overlap. The paired trajectories thus show stronger subspace alignment in the coherence group despite their different outcomes.

\subsection{Coherence-Channel Change under SFT and GRPO}
\label{sec:coherence_training_change}
We compare independent SFT and full-model GRPO branches initialized from the same pretrained model, with no SFT initialization for the GRPO branch. Figure~\ref{fig:analysis}(b) reports the Frobenius distance between each trained model's coherence-channel projector and the pretrained projector, estimated on shared calibration prompts. Writing $C^{(t)}$ for the coherence-channel set and $P_C^{(t)}=\operatorname{diag}(\mathbf{1}_{C^{(t)}})$ for its projector, we measure
\begin{equation}
\begin{aligned} D_C^{(t)}&=\left\|P_C^{(t)}-P_C^{(\mathrm{pre})}\right\|_F\\ &=\sqrt{\left|C^{(t)}\mathbin{\triangle}C^{(\mathrm{pre})}\right|}. \end{aligned}
\label{eq:coherence-membership-drift}
\end{equation}
Here $t\in\{\mathrm{SFT},\mathrm{GRPO}\}$, $\mathrm{pre}$ denotes pretraining, and $\triangle$ is symmetric set difference. SFT shows larger coherence-channel membership changes than direct GRPO for these checkpoints. Both branches modify the identified channels, but the direct-GRPO changes are smaller.

\subsection{Coherence versus Reasoning Changes during GRPO}
\label{sec:grpo_channel_change}
We next examine a separate full-model GRPO run initialized from SFT, keeping the initial channel partition fixed. Figure~\ref{fig:analysis}(c) shows the Frobenius norm of activation change, normalized by each group's initial activation norm, using identical completion tokens before and after GRPO. Let $H^{(0)}$ and $H^{(1)}$ be the initial and final activations, and $P_g^{(0)}$ the fixed initial group projector, with $P_R^{(0)}=I-P_C^{(0)}$. Raw and percentage-relative changes are
\begin{equation}
\begin{aligned} F_g&=\left\|(H^{(1)}-H^{(0)})P_g^{(0)}\right\|_F,\\ A_g&=100\,\frac{F_g}{\left\|H^{(0)}P_g^{(0)}\right\|_F}, \qquad g\in\{C,R\}. \end{aligned}
\label{eq:relative-channel-drift}
\end{equation}
The plot reports $A_g$. The complementary reasoning channels show larger relative changes, while coherence-channel changes remain smaller but nonzero. The shaded bands represent pointwise 95\% problem-bootstrap confidence intervals.

\section{Methodology}
\label{sec:methodology}

\begin{figure}[H]
    \centering
    \includegraphics[width=1\linewidth]{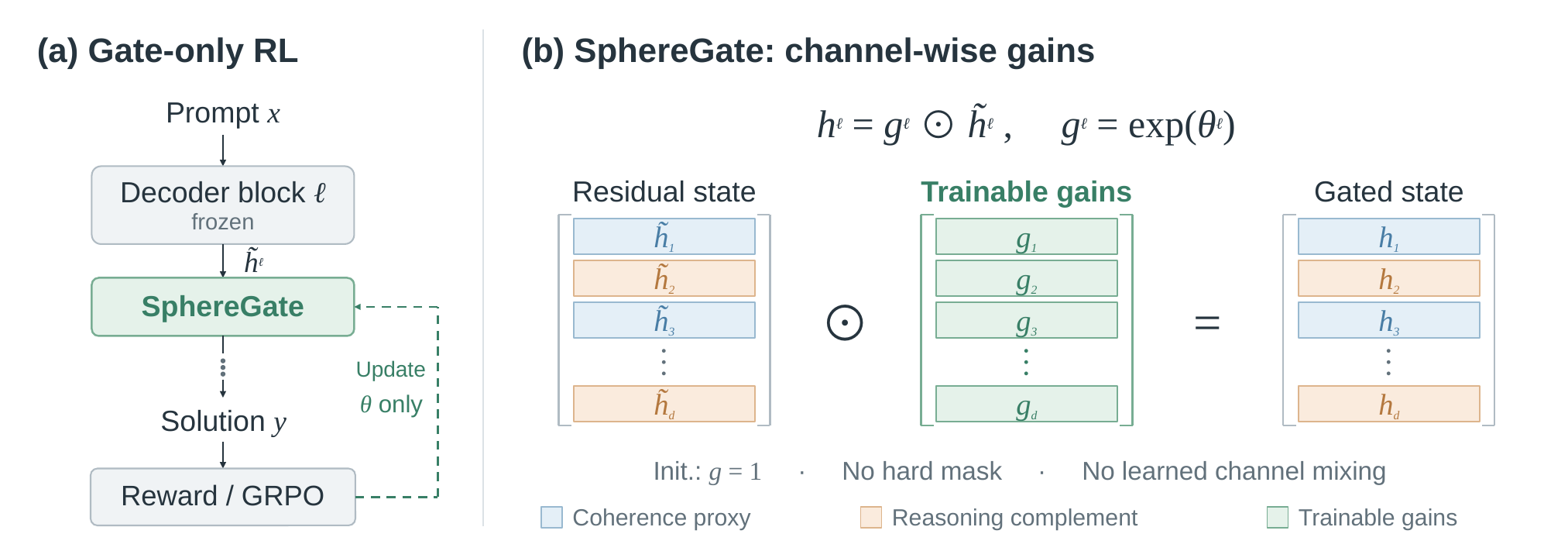}
    \caption{The General Framework for Our Proposed Method, SphereGate.}
    \label{fig:method}
\end{figure}

\subsection{SphereGate}
\label{sec:spheregate}

Consider a decoder block $l$ with residual input $h^{l-1}\in\mathbb{R}^d$. A standard Transformer block writes into the residual stream through an attention sublayer followed by an MLP sublayer:
\begin{align}
a^l &= \mathrm{Attn}^l(\mathrm{LN}(h^{l-1})), \label{eq:attn} \\
u^l &= h^{l-1} + a^l, \label{eq:res-attn} \\
m^l &= \mathrm{MLP}^l(\mathrm{LN}(u^l)), \label{eq:mlp} \\
\tilde{h}^l &= u^l + m^l, \label{eq:res-mlp}
\end{align}
where $\tilde{h}^l$ is the full residual state after the MLP write, and $\mathrm{LN}$ denotes layer normalization (modern models typically use RMSNorm). SphereGate then applies a learned per-channel gain directly to this residual state:
\begin{equation}
h^l = g^l \odot \tilde{h}^l, \qquad g^l = \exp(\theta^l). \label{eq:spheregate}
\end{equation}
During training, only $\theta^l$ is trainable; all backbone parameters remain frozen.

\subsection{Gradient Analysis}
\label{sec:gradient-analysis}

Let $L$ denote the number of decoder blocks and $d_\ell$
the residual dimension at block $\ell$. For token position
$t$, SphereGate maps the pre-gate state
$\widetilde{h}_{\ell t}\in\mathbb{R}^{d_\ell}$ to
\[
    h_{\ell t}
    =
    \exp(\theta_\ell)\odot\widetilde{h}_{\ell t},
\]
where $\theta_\ell$ contains the channel-wise log-gains.
Within their bounds $|\theta_{\ell k}|<c$, with $c>0$,
the gradient of a differentiable loss $\mathcal{L}$ is
\begin{equation}
    \nabla_{\theta_\ell}\mathcal{L}
    =
    \sum_{t\in\mathcal{T}}
    \frac{\partial\mathcal{L}}{\partial h_{\ell t}}
    \odot h_{\ell t},
    \label{eq:sg-log-gain-gradient}
\end{equation}
where $\mathcal{T}$ contains the nonpadding token
positions in the fixed training batch. Thus, gain
gradients aggregate activation-weighted loss sensitivities.

\paragraph{Channel partition.}
As mentioned in the figure \ref{fig:anisotropy-overview}, we identify anisotropic channels using a fixed calibration
pass. Let $h^{\mathrm{cal}}_{\ell tk}$ denote the post-block
activation of channel $k$ at calibration position $t$,
measured before final model normalization at the last block.
For the nonpadding positions
$\mathcal{T}_{\mathrm{cal}}$, with
$N_{\mathrm{cal}}=|\mathcal{T}_{\mathrm{cal}}|$, define
\[
    \mu_{\ell k}
    =
    \frac{1}{N_{\mathrm{cal}}}
    \sum_{t\in\mathcal{T}_{\mathrm{cal}}}
    h^{\mathrm{cal}}_{\ell tk},
    \qquad
    v_{\ell k}=|\mu_{\ell k}|.
\]
Using the median and unscaled median absolute deviation,
\[
    m_\ell=\operatorname{median}_k v_{\ell k},
    \qquad
    s_\ell=\operatorname{median}_k|v_{\ell k}-m_\ell|,
\]
we select
\[
    C_\ell=\{k:v_{\ell k}>m_\ell+6s_\ell\},
    \qquad
    R_\ell=\{1,\ldots,d_\ell\}\setminus C_\ell.
\]
Index all layer-channel pairs by $j\in\{1,\ldots,p\}$,
where $p=\sum_{\ell=1}^L d_\ell$. The sets $C$ and $R$
collect the selected and complementary coordinates
across layers and remain fixed during the analysis.
Default SphereGate training optimizes gains on all
coordinates.

\paragraph{Energy-weighted stability.}
Let $\pi_\theta$ denote the gate-adapted policy.
For a fixed rollout batch, consider the maximization objective
\[
    J(\theta)=U(\theta)-\beta K(\theta),
    \qquad \beta>0,
\]
where $U$ is the length-normalized GRPO reward surrogate
and $K$ is the KL penalty to a frozen reference policy
$\pi_{\mathrm{ref}}$. Samples, advantages, and sequence
lengths are fixed during this optimization stage.

Let $\theta_0$ be the rollout-policy log-gains,
$\theta_1$ a subsequent iterate, and
$\Delta=\theta_1-\theta_0$. Assume that the connecting
segment lies inside the gain bounds and that $J$ is
twice continuously differentiable on a neighborhood
of the segment. Define the baseline channel energies
and their diagonal matrix by
\[
    E_j=\sum_{t\in\mathcal{T}}h_{tj}(\theta_0)^2>0,
    \qquad
    D=\operatorname{diag}(E_1,\ldots,E_p),
\]
where $h_{tj}$ denotes the activation at the
layer-channel pair indexed by $j$.

Write $z=D^{1/2}\Delta$ for the energy-weighted
log-gain displacement. Define the selected group's
initial gradient magnitude and final gradient residual as
\begin{equation}
    \begin{aligned}
        \varepsilon_C
        &=
        \|D_C^{-1/2}\nabla_C J(\theta_0)\|_2,\\
        \tau_C
        &=
        \|D_C^{-1/2}\nabla_C J(\theta_1)\|_2.
    \end{aligned}
    \label{eq:sg-normalized-signals}
\end{equation}
Subscripts denote coordinate restrictions or matrix
blocks; in particular, $D_C$ is the principal block
indexed by $C$. The norm $\|\cdot\|_2$ is Euclidean
for vectors and spectral for matrices.

Define the integrated normalized curvature
\begin{equation}
    \mathcal{H}
    =
    D^{-1/2}
    \left[
        -\int_0^1
        \nabla^2J(\theta_0+s\Delta)\,ds
    \right]
    D^{-1/2}.
    \label{eq:sg-integrated-curvature}
\end{equation}

\begin{proposition}[Energy-weighted gate stability]
\label{prop:sg-conditional-stability}
Suppose that, for $\lambda>0$ and $\rho\geq0$,
\[
    \mathcal{H}_{CC}\succeq\lambda I_{|C|},
    \qquad
    \|\mathcal{H}_{CR}\|_2\leq\lambda\rho,
\]
where $\succeq$ denotes positive-semidefinite ordering
and $I_{|C|}$ is the identity on the selected coordinates.
Then
\begin{align}
    \|z_C\|_2
    &\leq
    B_C:=
    \frac{\varepsilon_C+\tau_C}{\lambda}
    +\rho\|z_R\|_2,
    \label{eq:sg-group-stability}\\
    |\Delta_j|
    &\leq
    \frac{B_C}{\sqrt{E_j}},
    \qquad j\in C.
    \label{eq:sg-coordinate-stability}
\end{align}
\end{proposition}

The bound separates selected-group optimization pressure
from cross-group coupling. For a fixed $B_C$, it provides
an inverse-square-root energy envelope on coordinate-wise
log-gain movement. Moreover, if $\|z_R\|_2>0$ and
\[
    \frac{\varepsilon_C+\tau_C}{\lambda}
    \leq\alpha\|z_R\|_2,
    \qquad
    \alpha\geq0,\quad \alpha+\rho<1,
\]
then
\[
    \|z_C\|_2
    \leq(\alpha+\rho)\|z_R\|_2
    <\|z_R\|_2.
\]
This gives a precise criterion for adaptation to
concentrate in the complementary group under the
energy-weighted metric. Appendix~\ref{app:sg-gate-stability}
provides the proof, the corresponding direct
representation-displacement bound, and explicit
connections to the GRPO reward signal and KL curvature.

\section{Results}
\label{sec:results}

\subsection{Experimental Setup}
\label{sec:setup}

\textbf{Models and data.} We evaluate \method{} on Qwen2.5-Instruct
(0.5B, 1.5B, 3B, and 7B), Llama-3-8B-Instruct
\cite{grattafiori2025llama}, and Qwen2.5-Math-7B \cite{yang2024qwen2}.
Training data comprise MATH-train \citep{hendrycks2021measuring},
DAPO-Math-17k \cite{yu2026dapo}, and MIX10k
(an equal mixture of 5k MATH and 5k DAPO).

\textbf{GRPO.} We use a PPO-clip surrogate with clip-eps $0.2$,
$\mu=2$ inner epochs, and group-normalized advantages
(std\_adv=True, len\_norm=True). The KL coefficient is $0.01$;
the gate-disabled model serves as the identity reference, avoiding
a second model copy. Training uses 90 steps, 64 prompts and
8 completions per prompt (512 rollouts per step), temperature 1.0,
top-p 1.0, and a 1024-token cap (4096 for the Math model).

\textbf{Baselines.} We compare with TinyLoRA \citep{morris2026learning},
LoRA-XS \citep{balazy2024lora}, random subnetwork adaptation
\citep{mukherjee2025reinforcement}, IA$^3$, LoRA \citep{hu2021lora},
and full fine-tuning. \method{}, LoRA-XS, TinyLoRA, and Subnetwork
each train $\approx$0.100M parameters, with learning rates of
1e-3, 1e-3, 2e-4, and 1e-4, respectively. \method{} trains all channels
with $\theta$ initialized to 0 and clamped at $c=0.7$;
LoRA-XS and TinyLoRA use SVD-rank 64, and Subnetwork uses
subnet-k 10000. IA$^3$ uses a learning rate of $1e-5$ and learned
scaling vectors on key, value, and feed-forward projections;
full fine-tuning uses 1e-6. All methods share prompts, rewards,
held-out splits, and rollout budgets.

\subsection{MATH-500 Performance}
\label{sec:math500}

\method{} achieves the highest reported greedy MATH-500 accuracy
on every backbone (Table~\ref{tab:math500}). On Qwen2.5-0.5B, its
$42.3$ exceeds the strongest evaluated PEFT baseline, LoRA-XS
($35.0$), by $7.3$ points; the corresponding 7B margin is $2.5$ points.
Accuracy improves over the base models from $64.5$ to $70.0$ on
Qwen2.5-3B, $70.0$ to $79.5$ on Qwen2.5-7B, and $33.5$ to $38.0$
on Llama-3-8B-Instruct.

\newcommand{\result}[2]{#1\,{\scriptsize$\pm$#2}}
\begin{table*}[t]
\centering
\small
\caption{Performance on MATH-500 under GRPO using greedy decoding.
All values are absolute accuracy scores on a $0$--$100$ scale. Standard deviations are estimated across three training seed checkpoints.}
\label{tab:math500}
\begin{tabular}{lcccc}
\toprule
\textbf{Method}
& \textbf{Qwen2.5-0.5B-It}
& \textbf{Qwen2.5-3B-It}
& \textbf{Qwen2.5-7B-It}
& \textbf{Llama-3-8B-It} \\
\midrule

Base Model
& \result{32.5}{0.82}
& \result{64.5}{0.71}
& \result{70.0}{0.80}
& \result{33.5}{0.76} \\

Full GRPO
& \result{41.6}{0.72}
& \result{69.2}{0.60}
& \result{77.2}{0.61}
& \result{37.4}{0.65} \\

IA$^3$
& \result{31.7}{1.02}
& \result{63.0}{0.88}
& \result{68.6}{0.94}
& \result{32.4}{0.93} \\

TinyLoRA
& \result{33.5}{0.74}
& \result{68.0}{0.63}
& \result{75.0}{0.58}
& \result{35.0}{0.72} \\

LoRA-XS
& \result{35.0}{0.68}
& \result{67.0}{0.57}
& \result{77.0}{0.52}
& \result{36.0}{0.66} \\

Subnetwork
& \result{34.5}{0.91}
& \result{67.0}{0.76}
& \result{73.0}{0.70}
& \result{34.0}{0.84} \\

\textbf{SphereGate}
& \textbf{\result{42.3}{0.61}}
& \textbf{\result{70.0}{0.54}}
& \textbf{\result{79.5}{0.46}}
& \textbf{\result{38.0}{0.58}} \\

\bottomrule
\end{tabular}
\end{table*}

\subsection{Mathematical Reasoning Performance}
\label{sec:math_benchmarks}

Table~\ref{tab:amc_minerva} reports mean accuracy over 16 sampled
responses (Avg@16) on AMC23 \cite{cao2025step} and Minerva Math
\cite{lewkowycz2022solving} for Qwen2.5-7B-Instruct and
Llama-3-8B-Instruct. \method{} leads the evaluated PEFT baselines
in all four settings. On Qwen2.5-7B AMC23, it scores $55.0\%$,
versus $53.7\%$ for Subnetwork and $51.9\%$ for the base model
(a $3.1$-point gain). On Llama-3-8B Minerva Math, it gains
$1.49$ points over the base model, compared with flat or lower
scores for TinyLoRA and LoRA-XS.

\begin{table*}[t]
\centering
\small
\caption{Mathematical reasoning performance on AMC23 and Minerva Math
using Avg@16. Standard deviations are estimated across three training seed checkpoints.}
\label{tab:amc_minerva}

\begin{tabular}{lcccc}
\toprule
\textbf{Method}
& \multicolumn{2}{c}{\textbf{AMC23}}
& \multicolumn{2}{c}{\textbf{Minerva Math}} \\
\cmidrule(lr){2-3}\cmidrule(lr){4-5}
& \textbf{Qwen2.5-7B-It}
& \textbf{Llama-3-8B-It}
& \textbf{Qwen2.5-7B-It}
& \textbf{Llama-3-8B-It} \\
\midrule

Base Model
& \result{51.9}{1.38}
& \result{8.8}{0.47}
& \result{33.66}{0.72}
& \result{12.39}{0.61} \\

Full GRPO
& \result{53.6}{1.00}
& \result{10.2}{0.34}
& \result{36.1}{0.70}
& \result{14.1}{0.70} \\

IA$^3$
& \result{50.8}{1.47}
& \result{8.4}{0.52}
& \result{32.90}{0.81}
& \result{11.80}{0.78} \\

LoRA-XS
& \result{52.8}{1.11}
& \result{9.2}{0.39}
& \result{34.85}{0.61}
& \result{12.34}{0.74} \\

TinyLoRA
& \result{52.2}{1.27}
& \result{9.4}{0.22}
& \result{34.73}{0.58}
& \result{11.08}{0.86} \\

Subnetwork
& \result{53.7}{1.13}
& \result{9.6}{0.28}
& \result{34.87}{0.86}
& \result{12.02}{0.88} \\

\textbf{SphereGate}
& \textbf{\result{55.0}{0.82}}
& \textbf{\result{10.0}{0.31}}
& \textbf{\result{36.28}{0.66}}
& \textbf{\result{13.88}{0.68}} \\

\bottomrule
\end{tabular}
\end{table*}

\subsection{Competition and Olympiad Benchmarks}
\label{sec:competition_benchmarks}

On Qwen2.5-7B-Math-Instruct, \method{} reaches $23.33\%$ on AIME24
and $44.0\%$ on OlympiadBench using maj@16 (Table~\ref{tab:aime_olympiad}),
supporting out-of-distribution transfer. On AIME24, it improves the
base score ($16.67\%$) by $6.66$ points and exceeds LoRA-XS
($20.00\%$), while TinyLoRA and Subnetwork match the base score.
\begin{table}[htbp]
\centering
\small
\caption{Transfer performance on AIME24 and OlympiadBench for
Qwen2.5-7B-Math-Instruct. AIME24 and OlympiadBench is evaluated using maj@16, All values are percentages.}
\label{tab:aime_olympiad}

\begin{tabular}{lcc}
\toprule
\textbf{Method}
& \textbf{AIME24}
& \textbf{OlympiadBench} \\
\midrule

Base Model
& 16.67
& 41.5 \\

Full GRPO
& 26.67
& 44.5 \\

IA$^3$
& 16.67
& 40.2 \\

LoRA-XS
& 20.00
& 43.0 \\

TinyLoRA
& 16.67
& 43.0 \\

Subnetwork
& 16.67
& 42.0 \\

\textbf{SphereGate}
& \textbf{23.33}
& \textbf{44.0} \\

\bottomrule
\end{tabular}
\end{table}

\subsection{Channel Ablations: Reasoning Adaptation and General-Text Prediction}
\label{sec:ablation}
\label{sec:perplexity-channel-ablation}

We distinguish restricted-gain training from inference-time activation
removal (Table~\ref{tab:ablation}). On Qwen2.5-3B-Instruct, excluded
gains remain fixed at one, preserving their activations. Complement-only
GRPO reaches 69.4 on MATH-500, within 0.6 percentage points of
all-channel \method{} (70.0).

For general-text prediction, we remove count-matched channel groups
from pretrained Qwen2.5-1.5B and its MATH-SFT checkpoint. Masks
estimated from 256 MATH training prompts select 4.19\% and 4.36\%
of coordinates on average, respectively. Each block's selected group
is compared with an equal number of uniformly sampled complementary
channels, averaging control perplexities across five random-mask seeds.
All conditions score the same 299,077 WikiText-2 test targets with
context length 1,024 and stride 512, without retraining. Per-token
residual-norm restoration to the value immediately before masking
additionally controls scale reduction. Anisotropic-channel removal
causes substantially greater prediction damage even with this control:
SFT perplexity rises from 10.44 to $2.01\times10^6$, versus 35.06
for matched complementary removal.

These findings are consistent with selected channels being important
for general-text prediction and complementary gains enabling effective
reasoning adaptation, supporting anisotropy as an operational coherence
proxy. They do not establish exclusive functional localization:
model sizes differ, restricted gates have unequal parameter counts,
and count matching with norm restoration equalizes neither removed
activation energy nor directional displacement. Perplexity measures
next-token prediction rather than semantic coherence directly.

\begin{table}[t]
\centering
\small

\begin{tabular}{lc}
\toprule
\multicolumn{2}{c}{
\textbf{(a) Restricted-gain training: Qwen2.5-3B-Instruct}
} \\
\midrule
Gate variant & MATH-500 greedy (\%) \\
\midrule
Baseline (un trained) & 64.5 \\
\method{} (all channels) & 70.0 \\
Anisotropic-only gate & 65.1 \\
Random Channel gate & 66.9 \\
Matched Complement-only gate & 67.5 \\
Complement-only gate & 69.4 \\
\bottomrule
\end{tabular}

\par\medskip

\begin{tabular}{llrr}
\toprule
\multicolumn{4}{c}{
\textbf{(b) Activation removal: Qwen2.5-1.5B}
} \\
\midrule
Removed channels & Scaling & Pretrained PPL & SFT PPL \\
\midrule
None & -- & 8.84 & 10.44 \\
Anisotropic & None
    & $3.91\times10^6$ & $2.44\times10^6$ \\
Matched complement & None
    & 31.18 & 37.96 \\
Anisotropic & Norm-restored
    & $2.73\times10^6$ & $2.01\times10^6$ \\
Matched complement & Norm-restored
    & 29.21 & 35.06 \\
\bottomrule
\end{tabular}

\caption{\textbf{Complementary channel ablations.}
(a) MATH-500 accuracy after GRPO with different trainable
channel groups; higher is better. Excluded gains remain at
one, so their activations are not removed.
(b) WikiText-2 perplexity under inference-time activation
removal; lower is better. Anisotropic and matched-complement
conditions remove the same number of coordinates in every
decoder block. Matched-complement perplexities are averaged
across five random-mask seeds. Full-test perplexities use
BF16 inference and FP64 loss accumulation. Norm restoration
is an ablation control, not part of the SphereGate training
operator.}
\label{tab:ablation}
\label{tab:perplexity-channel-ablation}
\end{table}
\section{Conclusion}
\label{sec:conclusion}

\method{} demonstrates that learning diagonal channel-wise gains is sufficient for effective GRPO adaptation of a frozen backbone, achieving strong mathematical reasoning performance across Qwen2.5 and Llama-3 models. Our analyses reveal a channel-level separation between coherence and reasoning: anisotropic residual channels predominantly support coherence, while complementary channels provide the primary capacity for reasoning adaptation.

\newpage
\subsection*{AI Use Statement}

In this work, we used generative AI tools, including ChatGPT, Claude, and Claude Code to implement methods, write and edit code, and assist in drafting the manuscript. We did not use generative AI tools to generate synthetic data, formulate , or propose our core analysis, methodology, or experimental design; these were carried out by the authors. We have reviewed all AI-assisted work: code was checked against expected behavior and the underlying data, and all AI-drafted text and result interpretations were verified against experimental outputs and revised where necessary. We take responsibility for the final content of this work, including text, claims, and artifacts produced with the aid of generative AI.

\bibliography{iclr2027_conference}

@article{morris2026learning,
  title={Learning to reason in 13 parameters},
  author={Morris, John X and Mireshghallah, Niloofar and Ibrahim, Mark and Mahloujifar, Saeed},
  journal={arXiv preprint arXiv:2602.04118},
  year={2026}
}

@article{balazy2024lora,
  title={Lora-xs: Low-rank adaptation with extremely small number of parameters},
  author={Ba{\l}azy, Klaudia and Banaei, Mohammadreza and Aberer, Karl and Tabor, Jacek},
  journal={arXiv preprint arXiv:2405.17604},
  year={2024}
}

@article{mukherjee2025reinforcement,
  title={Reinforcement learning finetunes small subnetworks in large language models, 2025},
  author={Mukherjee, Sagnik and Yuan, Lifan and Hakkani-Tur, Dilek and Peng, Hao},
  journal={URL https://arxiv. org/abs/2505.11711},
  year={2025}
}

@article{shao2024deepseekmath,
  title={Deepseekmath: Pushing the limits of mathematical reasoning in open language models},
  author={Shao, Zhihong and Wang, Peiyi and Zhu, Qihao and Xu, Runxin and Song, Junxiao and Bi, Xiao and Zhang, Haowei and Zhang, Mingchuan and Li, YK and Wu, Yang and others},
  journal={arXiv preprint arXiv:2402.03300},
  year={2024}
}

@article{yu2026dapo,
  title={Dapo: An open-source llm reinforcement learning system at scale},
  author={Yu, Qiying and Zhang, Zheng and Zhu, Ruofei and Yuan, Yufeng and Zuo, Xiaochen and Yue, Yu and Dai, Weinan and Fan, Tiantian and Liu, Gaohong and Liu, Lingjun and others},
  journal={Advances in Neural Information Processing Systems},
  volume={38},
  pages={113222--113244},
  year={2026}
}

@article{hu2021lora,
  title={Lora: Low-rank adaptation of large language models},
  author={Hu, Edward J and Shen, Yelong and Wallis, Phillip and Allen-Zhu, Zeyuan and Li, Yuanzhi and Wang, Shean and Wang, Lu and Chen, Weizhu},
  journal={arXiv preprint arXiv:2106.09685},
  year={2021}
}

@article{liu2022few,
  title={Few-shot parameter-efficient fine-tuning is better and cheaper than in-context learning},
  author={Liu, Haokun and Tam, Derek and Muqeeth, Mohammed and Mohta, Jay and Huang, Tenghao and Bansal, Mohit and Raffel, Colin A},
  journal={Advances in neural information processing systems},
  volume={35},
  pages={1950--1965},
  year={2022}
}

@article{hendrycks2021measuring,
  title={Measuring mathematical problem solving with the math dataset},
  author={Hendrycks, Dan and Burns, Collin and Kadavath, Saurav and Arora, Akul and Basart, Steven and Tang, Eric and Song, Dawn and Steinhardt, Jacob},
  journal={arXiv preprint arXiv:2103.03874},
  year={2021}
}

@article{ouyang2022training,
  title={Training language models to follow instructions with human feedback},
  author={Ouyang, Long and Wu, Jeffrey and Jiang, Xu and Almeida, Diogo and Wainwright, Carroll and Mishkin, Pamela and Zhang, Chong and Agarwal, Sandhini and Slama, Katarina and Ray, Alex and others},
  journal={Advances in neural information processing systems},
  volume={35},
  pages={27730--27744},
  year={2022}
}

@inproceedings{martinez2024mitigating,
  title={Mitigating frequency bias and anisotropy in language model pre-training with syntactic smoothing},
  author={Martinez, Richard Diehl and Goriely, Z{\'e}bulon and Caines, Andrew and Buttery, Paula and Beinborn, Lisa},
  booktitle={Proceedings of the 2024 Conference on Empirical Methods in Natural Language Processing},
  pages={5999--6011},
  year={2024}
}

@inproceedings{kim2025reasoning,
  title={Reasoning circuits in language models: A mechanistic interpretation of syllogistic inference},
  author={Kim, Geonhee and Valentino, Marco and Freitas, Andre},
  booktitle={Findings of the Association for Computational Linguistics: ACL 2025},
  pages={10074--10095},
  year={2025}
}

@article{arora2026language,
  title={Language model circuits are sparse in the neuron basis},
  author={Arora, Aryaman and Wu, Zhengxuan and Steinhardt, Jacob and Schwettmann, Sarah},
  journal={arXiv preprint arXiv:2601.22594},
  year={2026}
}

@inproceedings{godey2024anisotropy,
  title={Anisotropy is inherent to self-attention in transformers},
  author={Godey, Nathan and Clergerie, {\'E}ric and Sagot, Beno{\^\i}t},
  booktitle={Proceedings of the 18th Conference of the European Chapter of the Association for Computational Linguistics (Volume 1: Long Papers)},
  pages={35--48},
  year={2024}
}

@article{gao2019representation,
  title={Representation degeneration problem in training natural language generation models},
  author={Gao, Jun and He, Di and Tan, Xu and Qin, Tao and Wang, Liwei and Liu, Tie-Yan},
  journal={arXiv preprint arXiv:1907.12009},
  year={2019}
}

@article{zou2023representation,
  title={Representation engineering: A top-down approach to ai transparency},
  author={Zou, Andy and Phan, Long and Chen, Sarah and Campbell, James and Guo, Phillip and Ren, Richard and Pan, Alexander and Yin, Xuwang and Mazeika, Mantas and Dombrowski, Ann-Kathrin and others},
  journal={arXiv preprint arXiv:2310.01405},
  year={2023}
}

@article{turner2024steering,
  title={Steering language models with activation engineering, 2024},
  author={Turner, Alexander Matt and Thiergart, Lisa and Leech, Gavin and Udell, David and Vazquez, Juan J and Mini, Ulisse and MacDiarmid, Monte},
  journal={URL https://arxiv. org/abs/2308.10248},
  volume={2308},
  year={2024}
}

@article{panickssery2023steering,
  title={Steering llama 2 via contrastive activation addition, 2024},
  author={Panickssery, Nina and Gabrieli, Nick and Schulz, Julian and Tong, Meg and Hubinger, Evan and Turner, Alexander Matt},
  journal={URL https://arxiv. org/abs/2312.06681},
  volume={3},
  year={2023}
}

@article{wu2404reft,
  title={Reft: Representation finetuning for language models, 2024},
  author={Wu, Zhengxuan and Arora, Aryaman and Wang, Zheng and Geiger, Atticus and Jurafsky, Dan and Manning, Christopher D and Potts, Christopher},
  journal={URL https://arxiv. org/abs/2404.03592}
}

@article{du2026does,
  title={What Does an LLM Learn from Reinforcement Learning? A Mechanistic Interpretability Perspective with Fixed-SAE Track},
  author={Du, Lingheng and Tang, Yiming and Duan, Xufeng and Liu, Dianbo},
  journal={arXiv preprint arXiv:2609.15064},
  year={2026}
}

@article{olsson2022context,
  title={In-context learning and induction heads},
  author={Olsson, Catherine and Elhage, Nelson and Nanda, Neel and Joseph, Nicholas and DasSarma, Nova and Henighan, Tom and Mann, Ben and Askell, Amanda and Bai, Yuntao and Chen, Anna and others},
  journal={arXiv preprint arXiv:2209.11895},
  year={2022}
}

@article{tang2026capsule,
  title={Capsule Lens: Locating and Tracking Concept Geometry in Model Representations},
  author={Tang, Yiming and Saini, Harshvardhan and Jha, Samyak and Chen, Huaming and Duan, Xufeng and Liu, Dianbo},
  journal={arXiv preprint arXiv:2609.05575},
  year={2026}
}

@article{zhao2511rep2text,
  title={Rep2text: Decoding full text from a single llm token representation, 2026},
  author={Zhao, Haiyan and He, Zirui and Tang, Yiming and Yang, Fan and Payani, Ali and Liu, Dianbo and Du, Mengnan},
  journal={URL https://arxiv. org/abs/2511.06571},
  volume={2511},
  year={2026}
}

@inproceedings{
dai2025san,
title={{SAN}: Hypothesizing Long-Term Synaptic Development and Neural Engram Mechanism in Scalable Model's Parameter-Efficient Fine-Tuning},
author={Gaole Dai and Chun-Kai Fan and Yiming Tang and Zhi Zhang and Yuan Zhang and Yulu Gan and Qizhe Zhang and Cheng-Ching Tseng and Shanghang Zhang and Tiejun Huang},
booktitle={Forty-second International Conference on Machine Learning},
year={2025},
url={https://openreview.net/forum?id=CeTXRRAdFi}
}

@article{liu2024dora,
  title={Dora: Weight-decomposed low-rank adaptation},
  author={Liu, Shih-Yang and Wang, Chien-Yi and Yin, Hongxu and Molchanov, Pavlo and Wang, Yu-Chiang Frank and Cheng, Kwang-Ting and Chen, Min-Hung},
  journal={arXiv preprint arXiv:2402.09353},
  year={2024}
}

@article{yang2024qwen2,
  title={Qwen2. 5-math technical report: Toward mathematical expert model via self-improvement},
  author={Yang, An and Zhang, Beichen and Hui, Binyuan and Gao, Bofei and Yu, Bowen and Li, Chengpeng and Liu, Dayiheng and Tu, Jianhong and Zhou, Jingren and Lin, Junyang and others},
  journal={arXiv preprint arXiv:2409.12122},
  year={2024}
}

@article{grattafiori2025llama,
  title={The Llama 3 herd of models},
  author={Grattafiori, Aaron and Dubey, Abhimanyu and Jauhri, Abhinav and Pandey, Abhinav and Kadian, Abhishek and Al-Dahle, Ahmad and Letman, Aiesha and Mathur, Akhil and Schelten, Alan and Vaughan, Alex and others},
  year={2025}
}

@inproceedings{cao2025step,
  title={Step guided reasoning: Improving mathematical reasoning using guidance generation and step reasoning},
  author={Cao, Lang and Zou, Yingtian and Peng, Chao and Chen, Renhong and Ning, Wu and Li, Yitong},
  booktitle={Proceedings of the 2025 Conference on Empirical Methods in Natural Language Processing},
  pages={21112--21129},
  year={2025}
}

@article{lewkowycz2022solving,
  title={Solving quantitative reasoning problems with language models},
  author={Lewkowycz, Aitor and Andreassen, Anders and Dohan, David and Dyer, Ethan and Michalewski, Henryk and Ramasesh, Vinay and Slone, Ambrose and Anil, Cem and Schlag, Imanol and Gutman-Solo, Theo and others},
  journal={Advances in neural information processing systems},
  volume={35},
  pages={3843--3857},
  year={2022}
}
\bibliographystyle{iclr2027_conference}

\appendix
\section{Appendix}
\label{sec:appendix-theory}

This appendix provides the formal derivations and geometric proofs supporting the empirical analysis in Section~3 and the parameter-efficient adaptation dynamics of \textsc{SphereGate} in Section~4.

\subsection{Forward-KL Demonstration Coverage vs. Advantage-Weighted Optimization}
\label{app:objective-geometry}

Let $x \sim d$ be a prompt sampled from the task distribution, $q(y \mid x)$ be the demonstration distribution, and $\pi_\theta(y \mid x)$ be an autoregressive policy over completions $y \in \mathcal{Y}$. We assume a finite completion space, a fixed reference policy $\pi_{\mathrm{ref}}$ with full support, and a bounded reward function $R(x, y)$.

\begin{proposition}[SFT Minimizes Expected Forward-KL]
\label{prop:sft-forward-kl}
The sequence negative log-likelihood loss satisfies
\begin{equation}
\mathcal{L}_{\mathrm{SFT}}(\theta) 
= -\mathbb{E}_{x \sim d,\, y \sim q}\bigl[\log \pi_\theta(y \mid x)\bigr] 
= H_q(Y \mid X) + \mathbb{E}_{x \sim d}\bigl[D_{\mathrm{KL}}\bigl(q(\cdot \mid x) \,\|\, \pi_\theta(\cdot \mid x)\bigr)\bigr].
\label{eq:app-sft-forward-kl}
\end{equation}
Because the conditional data entropy $H_q(Y \mid X)$ is invariant to $\theta$, maximum likelihood estimation under teacher forcing is strictly equivalent to minimizing the expected forward Kullback--Leibler divergence.
\end{proposition}

\begin{proof}
By identity, $-\log \pi_\theta(y \mid x) = -\log q(y \mid x) + \log \frac{q(y \mid x)}{\pi_\theta(y \mid x)}$. Taking expectations with respect to $x \sim d$ and $y \sim q(\cdot \mid x)$ yields Equation~\eqref{eq:app-sft-forward-kl}. Since $\log \pi_\theta(y \mid x) = \sum_{t=1}^{|y|} \log \pi_\theta(y_t \mid x, y_{<t})$, the token-level cross-entropy loss sums directly to this sequence divergence.
\end{proof}

\begin{proposition}[Forward-KL Strongly Penalizes Dropping Empirical Modes]
\label{prop:forward-kl-covering}
Fix a prompt $x$ and let $S \subset \mathcal{Y}$ be any completion subset. Let $a = q(S \mid x) \in (0, 1)$ and $b = \pi_\theta(S \mid x) \in (0, 1)$. Then:
\begin{equation}
D_{\mathrm{KL}}\bigl(q(\cdot \mid x) \,\|\, \pi_\theta(\cdot \mid x)\bigr) \geq a \log \frac{a}{b} + (1 - a) \log \frac{1 - a}{1 - b}.
\label{eq:app-coverage-bound}
\end{equation}
Consequently, for any data-supported region ($a > 0$), $D_{\mathrm{KL}}(q \,\|\, \pi_\theta) \to +\infty$ as $b \to 0$.
\end{proposition}

\begin{proof}
Suppressing conditioning on $x$, decompose the divergence across $S$ and $S^c$:
\begin{align}
D_{\mathrm{KL}}(q \,\|\, \pi_\theta) 
&= a \log \frac{a}{b} + (1 - a) \log \frac{1 - a}{1 - b} \nonumber \\
&\quad + a D_{\mathrm{KL}}\bigl(q(\cdot \mid S) \,\|\, \pi_\theta(\cdot \mid S)\bigr) 
+ (1 - a) D_{\mathrm{KL}}\bigl(q(\cdot \mid S^c) \,\|\, \pi_\theta(\cdot \mid S^c)\bigr).
\end{align}
Nonnegativity of the conditional KL divergences gives the lower bound in Equation~\eqref{eq:app-coverage-bound}. The limit follows immediately since $\lim_{b \to 0^+} a \log(a/b) = +\infty$.
\end{proof}

\begin{remark}[Scope of Forward-KL]
Proposition~\ref{prop:forward-kl-covering} proves that SFT places an infinite penalty on omitting demonstration modes. This explains why SFT fits demonstration frequencies rather than explicitly distinguishing correctness, establishing the shared coherence substrate.
\end{remark}

\begin{proposition}[KL-Regularized Policy Optimization as Reverse-KL Projection]
\label{prop:rl-reverse-kl}
Consider the sequence-level KL-regularized reward objective:
\begin{equation}
J_\beta(\pi) = \mathbb{E}_{x \sim d}\left[\mathbb{E}_{y \sim \pi(\cdot \mid x)}[R(x, y)] - \beta D_{\mathrm{KL}}\bigl(\pi(\cdot \mid x) \,\|\, \pi_{\mathrm{ref}}(\cdot \mid x)\bigr)\right], \quad \beta > 0.
\label{eq:app-regularized-reward}
\end{equation}
Define the partition function $Z_\beta(x) = \sum_y \pi_{\mathrm{ref}}(y \mid x) \exp(R(x, y)/\beta)$ and the tilted target distribution $\pi_\beta^*(y \mid x) = \frac{1}{Z_\beta(x)} \pi_{\mathrm{ref}}(y \mid x) \exp(R(x, y)/\beta)$. Then:
\begin{equation}
J_\beta(\pi) = \beta \mathbb{E}_{x \sim d}\bigl[\log Z_\beta(x)\bigr] - \beta \mathbb{E}_{x \sim d}\bigl[D_{\mathrm{KL}}\bigl(\pi(\cdot \mid x) \,\|\, \pi_\beta^*(\cdot \mid x)\bigr)\bigr].
\label{eq:app-rl-reverse-kl}
\end{equation}
\end{proposition}

\begin{proof}
Expanding the divergence between $\pi$ and $\pi_\beta^*$ for a fixed prompt $x$:
\begin{align}
D_{\mathrm{KL}}(\pi \,\|\, \pi_\beta^*) 
&= \mathbb{E}_{y \sim \pi}\left[\log \frac{\pi(y \mid x)}{\pi_{\mathrm{ref}}(y \mid x)} - \frac{R(x, y)}{\beta} + \log Z_\beta(x)\right] \nonumber \\
&= D_{\mathrm{KL}}(\pi \,\|\, \pi_{\mathrm{ref}}) - \frac{1}{\beta} \mathbb{E}_{y \sim \pi}[R(x, y)] + \log Z_\beta(x).
\end{align}
Multiplying through by $-\beta$ and taking expectations over $x \sim d$ yields Equation~\eqref{eq:app-rl-reverse-kl}.
\end{proof}

\paragraph{Surrogate Gradient of GRPO.}
In practical Group Relative Policy Optimization \citep{shao2024deepseekmath}, $G \ge 2$ trajectories $\{y_i\}_{i=1}^G$ are sampled from the current rollout checkpoint $\mu = \pi_{\theta_{\mathrm{old}}}$. Evaluating the token-averaged PPO-clip surrogate at $\theta = \theta_{\mathrm{old}}$, the policy ratio equals $1$ and clipping is inactive, yielding the local gradient:
\begin{equation}
\left.\nabla_\theta \widehat{J}_{\mathrm{rew}}(\theta)\right|_{\theta_{\mathrm{old}}} = \left.\frac{1}{G} \sum_{i=1}^G \frac{\widehat{A}_i}{|y_i|} \nabla_\theta \log \pi_\theta(y_i \mid x)\right|_{\theta_{\mathrm{old}}},
\label{eq:app-grpo-local-gradient}
\end{equation}
where $\widehat{A}_i = \frac{R_i - \bar{R}}{s_R + \delta}$ is the group-standardized advantage ($s_R$ being the sample standard deviation of rewards).

\subsection{Reward-Selective Gradients under Group Advantage Centering}
\label{app:centered-gradients}

Let $s_i = \left.\frac{1}{|y_i|} \nabla_\theta \log \pi_\theta(y_i \mid x)\right|_{\theta_{\mathrm{old}}}$ be the length-normalized score vector of completion $i$, with group mean $\bar{s} = \frac{1}{G} \sum_{i=1}^G s_i$. The unregularized reward gradient is $g_{\mathrm{rew}} = \frac{1}{G} \sum_{i=1}^G \widehat{A}_i s_i$.

\begin{proposition}[Centered Update as Reward--Score Covariance]
\label{prop:reward-score-covariance}
The local reward gradient under group normalization is proportional to the sample covariance between scalar rewards and score vectors:
\begin{equation}
g_{\mathrm{rew}} = \frac{1}{G(s_R + \delta)} \sum_{i=1}^G (R_i - \bar{R})(s_i - \bar{s}).
\label{eq:app-reward-score-covariance}
\end{equation}
\end{proposition}

\begin{proof}
Substituting $\widehat{A}_i = (R_i - \bar{R})/(s_R + \delta)$:
\begin{equation}
g_{\mathrm{rew}} = \frac{1}{G(s_R + \delta)} \sum_{i=1}^G (R_i - \bar{R}) s_i.
\end{equation}
Because the centered rewards sum to zero ($\sum_{i=1}^G (R_i - \bar{R}) = 0$), we have $\sum_{i=1}^G (R_i - \bar{R}) \bar{s} = 0$. Subtracting this term yields Equation~\eqref{eq:app-reward-score-covariance}.
\end{proof}

\begin{corollary}[Shared Score Invariance]
\label{cor:score-cancellation}
Decompose each completion score into a component shared across the rollout group and a completion-specific residual: $s_i = c + d_i$. For any parameter direction $v$, if $v^\top s_i$ is invariant across all completions in the group ($v^\top d_i = 0$), then:
\begin{equation}
v^\top g_{\mathrm{rew}} = 0.
\label{eq:app-common-direction-cancellation}
\end{equation}
\end{corollary}

\begin{proof}
Projecting Equation~\eqref{eq:app-reward-score-covariance} onto $v$:
\begin{equation}
v^\top g_{\mathrm{rew}} = \frac{1}{G(s_R + \delta)} \sum_{i=1}^G (R_i - \bar{R}) \bigl(v^\top s_i - v^\top \bar{s}\bigr).
\end{equation}
If $v^\top s_i = v^\top c$ for all $i \in \{1, \dots, G\}$, then $v^\top \bar{s} = v^\top c$. Every term in the sum vanishes identically.
\end{proof}

\begin{remark}[Scope of the Covariance Form]
Corollary~\ref{cor:score-cancellation} provides an objective-level motivation for why coherence representations remain stable during RL: features common to both successful and unsuccessful completions produce zero reward gradient under group centering. Note that shared activations do not strictly imply identical score vectors (which also depend on downstream sensitivities), but this establishes a strong structural bias against modifying universally shared sequence components.
\end{remark}

\subsection{Local Geometry and Selective Adaptation of Residual Gates}
\label{app:sg-gate-stability}

We derive the geometry of positive residual gates,
prove Proposition~\ref{prop:sg-conditional-stability},
and express its stability conditions through the
GRPO reward signal and the KL-regularized objective.
Throughout, $\|\cdot\|_F$ denotes the Frobenius norm.

\paragraph{Positive residual gates.}
For a matrix of pre-gate token activations
$\widetilde{H}_\ell$, the layer-$\ell$ gate gives
\[
    H_\ell=\widetilde{H}_\ell G_\ell,
    \qquad
    G_\ell=\operatorname{diag}\bigl(\exp(\theta_\ell)\bigr).
\]
At initialization, $\theta_\ell=0$ and $G_\ell=I$.
For $|\theta_{\ell k}|\leq c$, every diagonal entry
lies in $[e^{-c},e^c]$. Hence $G_\ell$ is invertible,
and the gate preserves the entrywise signs, zero
pattern, and rank of its pre-gate input. Furthermore,
\begin{equation}
    \begin{aligned}
        e^{-c}\|\widetilde{H}_\ell\|_F
        &\leq\|H_\ell\|_F
        \leq e^c\|\widetilde{H}_\ell\|_F,\\
        \|H_\ell-\widetilde{H}_\ell\|_F
        &\leq(e^c-1)\|\widetilde{H}_\ell\|_F.
    \end{aligned}
    \label{eq:sg-app-positive-gate-bounds}
\end{equation}
These inequalities follow by applying the scalar
gain bounds to each column.

Within the interior of the log-gain bounds,
\[
    \frac{\partial h_{\ell tk}}
         {\partial\theta_{\ell k}}
    =
    \widetilde{h}_{\ell tk}e^{\theta_{\ell k}}
    =
    h_{\ell tk}.
\]
The chain rule, summed over token positions, therefore
gives
\[
    \frac{\partial\mathcal{L}}
         {\partial\theta_{\ell k}}
    =
    \sum_{t\in\mathcal{T}}
    \frac{\partial\mathcal{L}}
         {\partial h_{\ell tk}}
    h_{\ell tk},
\]
establishing Equation~\eqref{eq:sg-log-gain-gradient}.

\paragraph{Direct gate geometry.}
Let $\Delta=\theta_1-\theta_0$. Evaluate the direct
gate intervention by fixing each pre-gate input
at its value under $\theta_0$. For the flattened
layer-channel index $j$,
\begin{equation}
    \Delta h^{\mathrm{direct}}_{tj}
    =
    h_{tj}(\theta_0)\bigl(e^{\Delta_j}-1\bigr).
    \label{eq:sg-app-direct-intervention}
\end{equation}
Consequently, for any coordinate group $g$,
\begin{equation}
    \|\Delta H_g^{\mathrm{direct}}\|_F^2
    =
    \sum_{j\in g}E_j\bigl(e^{\Delta_j}-1\bigr)^2,
    \label{eq:sg-app-exact-displacement}
\end{equation}
where $E_j=\sum_{t\in\mathcal{T}}h_{tj}(\theta_0)^2$.
For groups spanning multiple layers, the squared
norm sums over their layer-token-channel entries.

As $\Delta_g\to0$, Taylor expansion yields
\begin{equation}
    \|\Delta H_g^{\mathrm{direct}}\|_F^2
    =
    \Delta_g^\top D_g\Delta_g
    +O\bigl(\|\Delta_g\|_2^3\bigr).
    \label{eq:sg-app-local-metric}
\end{equation}
Thus, $D$ is the local quadratic metric of direct
gate-induced representation displacement.

A finite-displacement bound follows from the mean
value theorem. Define
$\delta_g=\|\Delta_g\|_\infty
=\max_{j\in g}|\Delta_j|$.
For every $|u|\leq\delta_g$,
\[
    e^{-\delta_g}|u|
    \leq|e^u-1|
    \leq e^{\delta_g}|u|.
\]
Substitution into
Equation~\eqref{eq:sg-app-exact-displacement} gives
\begin{equation}
    e^{-\delta_g}\|D_g^{1/2}\Delta_g\|_2
    \leq
    \|\Delta H_g^{\mathrm{direct}}\|_F
    \leq
    e^{\delta_g}\|D_g^{1/2}\Delta_g\|_2.
    \label{eq:sg-app-two-sided-displacement}
\end{equation}

\paragraph{Proof of energy-weighted gate stability.}

\begin{proof}[Proof of Proposition~\ref{prop:sg-conditional-stability}]
Define the full normalized gradient vectors
\[
    a=D^{-1/2}\nabla J(\theta_0),
    \qquad
    e=D^{-1/2}\nabla J(\theta_1).
\]
By Equation~\eqref{eq:sg-normalized-signals},
$\|a_C\|_2=\varepsilon_C$ and
$\|e_C\|_2=\tau_C$.

The fundamental theorem of calculus gives
\[
    \nabla J(\theta_1)-\nabla J(\theta_0)
    =
    \left[
        \int_0^1
        \nabla^2J(\theta_0+s\Delta)\,ds
    \right]\Delta.
\]
Multiplying by $D^{-1/2}$ and using
$z=D^{1/2}\Delta$ yields
\begin{equation}
    \mathcal{H}z=a-e.
    \label{eq:sg-app-gradient-displacement}
\end{equation}
Taking the selected-coordinate block,
\[
    \mathcal{H}_{CC}z_C
    =
    a_C-e_C-\mathcal{H}_{CR}z_R.
\]
Since $\mathcal{H}_{CC}\succeq\lambda I_{|C|}$
with $\lambda>0$,
\[
    \|\mathcal{H}_{CC}^{-1}\|_2
    \leq\lambda^{-1}.
\]
It follows that
\begin{align*}
    \|z_C\|_2
    &\leq
    \|\mathcal{H}_{CC}^{-1}\|_2
    \left(
        \|a_C\|_2+\|e_C\|_2
        +\|\mathcal{H}_{CR}\|_2\|z_R\|_2
    \right)\\
    &\leq
    \frac{\varepsilon_C+\tau_C}{\lambda}
    +\rho\|z_R\|_2
    =B_C.
\end{align*}
This proves Equation~\eqref{eq:sg-group-stability}.
For every $j\in C$,
\[
    |\Delta_j|
    =
    \frac{|z_j|}{\sqrt{E_j}}
    \leq
    \frac{\|z_C\|_2}{\sqrt{E_j}}
    \leq
    \frac{B_C}{\sqrt{E_j}},
\]
which establishes
Equation~\eqref{eq:sg-coordinate-stability}.
\end{proof}

\paragraph{Selective adaptation and representation stability.}
Suppose $\|z_R\|_2>0$ and
\[
    \frac{\varepsilon_C+\tau_C}{\lambda}
    \leq\alpha\|z_R\|_2,
    \qquad
    \alpha\geq0,\quad q:=\alpha+\rho<1.
\]
Then Proposition~\ref{prop:sg-conditional-stability}
gives
\begin{equation}
    \|z_C\|_2
    \leq q\|z_R\|_2
    <\|z_R\|_2.
    \label{eq:sg-app-selective-adaptation}
\end{equation}
The selected group therefore undergoes smaller
aggregate energy-weighted log-gain movement.

Combining the same proposition with
Equation~\eqref{eq:sg-app-two-sided-displacement}
gives the direct representation bound
\begin{equation}
    \|\Delta H_C^{\mathrm{direct}}\|_F
    \leq
    e^{\delta_C}
    \left(
        \frac{\varepsilon_C+\tau_C}{\lambda}
        +\rho\|z_R\|_2
    \right).
    \label{eq:sg-app-representation-stability}
\end{equation}
Moreover,
\begin{equation}
    \frac{\|\Delta H_C^{\mathrm{direct}}\|_F}
         {\|\Delta H_R^{\mathrm{direct}}\|_F}
    \leq e^{\delta_C+\delta_R}q.
    \label{eq:sg-app-direct-selectivity}
\end{equation}
Thus, $e^{\delta_C+\delta_R}q<1$ establishes smaller
direct representation displacement in the selected
group. At an interior stationary endpoint,
$\nabla J(\theta_1)=0$ and $\tau_C=0$.

\paragraph{The fixed-batch GRPO reward surrogate.}
Let a rollout batch contain $B$ prompts $x_b$ and
$G\geq2$ completions $y_{bi}$ per prompt, with
completion lengths $T_{bi}$ and rewards $r_{bi}$.
Define the group mean reward and normalized advantages
by
\[
    \bar r_b=\frac1G\sum_{i=1}^G r_{bi},
    \qquad
    A_{bi}
    =
    \frac{r_{bi}-\bar r_b}
         {s_{r,b}+\epsilon_{\mathrm{adv}}},
\]
where $s_{r,b}$ is the group reward standard deviation
used by the algorithm and $\epsilon_{\mathrm{adv}}>0$
is its stabilizing constant.

For completion token $t$, define the policy ratio
\[
    \varrho_{bit}(\theta)
    =
    \frac{
        \pi_\theta(y_{bi,t}\mid x_b,y_{bi,<t})
    }{
        \pi_{\theta_0}(y_{bi,t}\mid x_b,y_{bi,<t})
    }.
\]
With clipping width $\epsilon_{\mathrm{clip}}\in(0,1)$,
the length-normalized reward surrogate is
\begin{equation}
    \begin{aligned}
        U(\theta)
        &=
        \frac{1}{BG}
        \sum_{b=1}^B\sum_{i=1}^G
        \frac{1}{T_{bi}}
        \sum_{t=1}^{T_{bi}}
        \min\Bigl\{
            \varrho_{bit}(\theta)A_{bi},\\
        &\hspace{35mm}
            \operatorname{clip}\bigl(
                \varrho_{bit}(\theta),
                1-\epsilon_{\mathrm{clip}},
                1+\epsilon_{\mathrm{clip}}
            \bigr)A_{bi}
        \Bigr\}.
    \end{aligned}
    \label{eq:sg-app-grpo-surrogate}
\end{equation}
All sampled tokens and advantages remain fixed
when differentiating this surrogate.

At $\theta_0$, every ratio equals one.
Define the length-normalized completion score
\[
    s_{bi}
    =
    \left.
    \frac{1}{T_{bi}}
    \nabla_\theta
    \log\pi_\theta(y_{bi}\mid x_b)
    \right|_{\theta=\theta_0},
    \qquad
    \bar s_b=\frac1G\sum_i s_{bi}.
\]
The reward gradient contributed by prompt $b$ is
\begin{equation}
    \begin{aligned}
        g_b^{\mathrm{rew}}
        &=
        \frac1G\sum_i A_{bi}s_{bi}\\
        &=
        \frac{1}{G(s_{r,b}+\epsilon_{\mathrm{adv}})}
        \sum_i
        (r_{bi}-\bar r_b)(s_{bi}-\bar s_b).
    \end{aligned}
    \label{eq:sg-app-reward-covariance}
\end{equation}
The second equality follows from
$\sum_i(r_{bi}-\bar r_b)=0$.
Consequently,
$\nabla U(\theta_0)=B^{-1}\sum_b g_b^{\mathrm{rew}}$.

\paragraph{Bounding the normalized reward signal.}
Define empirical variances with denominator $G$:
\[
    \sigma_{r,b}^2
    =
    \frac1G\sum_i(r_{bi}-\bar r_b)^2,
    \qquad
    \sigma_{s,b,j}^2
    =
    \frac1G\sum_i(s_{bi,j}-\bar s_{b,j})^2.
\]
Applying Cauchy--Schwarz to
Equation~\eqref{eq:sg-app-reward-covariance} gives
\[
    |g_{b,j}^{\mathrm{rew}}|
    \leq
    \frac{\sigma_{r,b}}
         {s_{r,b}+\epsilon_{\mathrm{adv}}}
    \sigma_{s,b,j}.
\]
Squaring, dividing by $E_j$, and summing over
$j\in C$ yields
\[
    \|D_C^{-1/2}g_{b,C}^{\mathrm{rew}}\|_2
    \leq
    \frac{\sigma_{r,b}}
         {s_{r,b}+\epsilon_{\mathrm{adv}}}
    \left(
        \sum_{j\in C}
        \frac{\sigma_{s,b,j}^2}{E_j}
    \right)^{1/2}.
\]
Therefore, defining
\begin{equation}
    \varepsilon_C^{\mathrm{rew}}
    =
    \frac1B\sum_{b=1}^B
    \frac{\sigma_{r,b}}
         {s_{r,b}+\epsilon_{\mathrm{adv}}}
    \left(
        \sum_{j\in C}
        \frac{\sigma_{s,b,j}^2}{E_j}
    \right)^{1/2},
    \label{eq:sg-app-reward-signal-bound}
\end{equation}
the triangle inequality gives
\[
    \|D_C^{-1/2}\nabla_C U(\theta_0)\|_2
    \leq\varepsilon_C^{\mathrm{rew}}.
\]
Since $J=U-\beta K$,
\begin{equation}
    \varepsilon_C
    \leq
    \varepsilon_C^{\mathrm{rew}}
    +
    \beta
    \|D_C^{-1/2}\nabla_C K(\theta_0)\|_2.
    \label{eq:sg-app-full-signal-bound}
\end{equation}
This expresses the initial normalized optimization
signal in terms of reward-correlated score variation
and the reference-policy gradient. A score coordinate
shared across all completions in a group has
$\sigma_{s,b,j}=0$ and contributes zero reward gradient.

\paragraph{Sufficient curvature conditions.}
The curvature constants in the stability theorem
can be obtained from the reward and KL terms
separately. For $s\in[0,1]$, define
\[
    A_K(s)
    =
    D^{-1/2}
    \nabla^2K(\theta_0+s\Delta)
    D^{-1/2},
    \qquad
    A_U(s)
    =
    D^{-1/2}
    \nabla^2U(\theta_0+s\Delta)
    D^{-1/2}.
\]
Suppose that throughout the segment,
\begin{equation}
    \begin{aligned}
        [A_K(s)]_{CC}
        &\succeq\kappa I_{|C|},
        &
        [A_U(s)]_{CC}
        &\preceq\gamma I_{|C|},\\
        \|[A_K(s)]_{CR}\|_2
        &\leq\kappa_\times,
        &
        \|[A_U(s)]_{CR}\|_2
        &\leq\gamma_\times,
    \end{aligned}
    \label{eq:sg-app-component-curvature}
\end{equation}
where $\kappa>0$,
$\gamma,\kappa_\times,\gamma_\times\geq0$,
and $\beta\kappa>\gamma$.
Because
\[
    \mathcal{H}
    =
    \int_0^1
    \bigl(\beta A_K(s)-A_U(s)\bigr)\,ds,
\]
integration and the triangle inequality give
\[
    \mathcal{H}_{CC}
    \succeq(\beta\kappa-\gamma)I_{|C|},
    \qquad
    \|\mathcal{H}_{CR}\|_2
    \leq\beta\kappa_\times+\gamma_\times.
\]
Thus, Proposition~\ref{prop:sg-conditional-stability}
applies with
\begin{equation}
    \lambda=\beta\kappa-\gamma,
    \qquad
    \rho=
    \frac{\beta\kappa_\times+\gamma_\times}
         {\beta\kappa-\gamma}.
    \label{eq:sg-app-curvature-constants}
\end{equation}
These constants quantify the restoring curvature
and cross-group coupling of the complete
KL-regularized objective.

\paragraph{Curvature at the reference policy.}
For a fixed prompt distribution $d$ and a smooth,
full-support policy on a finite completion space,
consider the sequence-level KL
\[
    K_{\mathrm{seq}}(\theta)
    =
    \mathbb{E}_{x\sim d}
    D_{\mathrm{KL}}\bigl(
        \pi_\theta(\cdot\mid x)
        \,\|\,\pi_{\mathrm{ref}}(\cdot\mid x)
    \bigr).
\]
Let $\theta_{\mathrm{ref}}$ satisfy
$\pi_{\theta_{\mathrm{ref}}}=\pi_{\mathrm{ref}}$.
Normalization of $\pi_\theta$ gives
\[
    \nabla K_{\mathrm{seq}}(\theta)
    =
    \mathbb{E}_{x\sim d}
    \sum_y
    \nabla\pi_\theta(y\mid x)
    \log
    \frac{\pi_\theta(y\mid x)}
         {\pi_{\mathrm{ref}}(y\mid x)}.
\]
At $\theta_{\mathrm{ref}}$, the logarithm vanishes,
so $\nabla K_{\mathrm{seq}}(\theta_{\mathrm{ref}})=0$.
Differentiating once more yields
\begin{equation}
    \nabla^2K_{\mathrm{seq}}(\theta_{\mathrm{ref}})
    =
    \mathbb{E}_{x\sim d,\;y\sim\pi_{\mathrm{ref}}}
    [q(x,y)q(x,y)^\top],
    \label{eq:sg-app-fisher-curvature}
\end{equation}
where
\[
    q(x,y)
    =
    \left.
    \nabla_\theta\log\pi_\theta(y\mid x)
    \right|_{\theta=\theta_{\mathrm{ref}}}.
\]
Hence the reference-policy curvature is the
Fisher matrix.

The sampled token-level $k_3$ penalty has an
analogous empirical identity. For fixed contexts
$\xi_n$, tokens $a_n$, and positive weights
$\omega_n$ summing to one, define
\[
    \chi_n(\theta)
    =
    \frac{\pi_{\mathrm{ref}}(a_n\mid\xi_n)}
         {\pi_\theta(a_n\mid\xi_n)},
    \qquad
    K_{k_3}(\theta)
    =
    \sum_n\omega_n
    \bigl(\chi_n(\theta)-\log\chi_n(\theta)-1\bigr).
\]
Writing
$q_n(\theta)=
 \nabla_\theta\log\pi_\theta(a_n\mid\xi_n)$,
differentiation gives
\[
    \nabla K_{k_3}(\theta)
    =
    \sum_n\omega_n
    \bigl(1-\chi_n(\theta)\bigr)q_n(\theta).
\]
At the reference-policy anchor, $\chi_n=1$, and
therefore
\begin{equation}
    \begin{aligned}
        \nabla K_{k_3}(\theta_{\mathrm{ref}})
        &=0,\\
        \nabla^2K_{k_3}(\theta_{\mathrm{ref}})
        &=
        \sum_n\omega_n
        q_n(\theta_{\mathrm{ref}})
        q_n(\theta_{\mathrm{ref}})^\top.
    \end{aligned}
    \label{eq:sg-app-sampled-kl-curvature}
\end{equation}
Energy normalization transforms these Fisher and
empirical Gram matrices into Gram matrices of
energy-normalized policy scores.

For either penalty, when
$\theta_0=\theta_{\mathrm{ref}}$,
Equation~\eqref{eq:sg-app-full-signal-bound} becomes
\[
    \varepsilon_C
    \leq\varepsilon_C^{\mathrm{rew}}.
\]
Together, the reward-signal bound, the segment-wise
curvature conditions, and the direct gate geometry
establish an explicit criterion for selective
adaptation in the energy-weighted representation metric.

\subsection{Channel Subspace Metrics and Calibrations}
\label{app:subspace-geometry}

This section provides formal proofs for the empirical subspace metrics evaluated in Section~3.

\begin{proposition}[Overlap as Mean-Squared Principal Cosine]
\label{prop:overlap-derivation}
Let $U_g^+ \in \mathbb{R}^{D \times r_g^+}$ and $U_g^- \in \mathbb{R}^{D \times r_g^-}$ denote orthonormal bases for the correct and wrong representation subspaces in channel group $g \in \{C, R\}$, with $k = \min(r_g^+, r_g^-)$. The overlap metric defined in Equation~(2) of the main text satisfies:
\begin{equation}
\mathcal{O}_g^2 = \frac{\|(U_g^+)^\top U_g^-\|_F^2}{k} = \frac{1}{k} \sum_{j=1}^k \cos^2 \phi_j = \frac{\operatorname{tr}(P_g^+ P_g^-)}{k}, \qquad \mathcal{O}_g \in [0, 1],
\end{equation}
where $\phi_1 \le \dots \le \phi_k$ are the principal angles between the two subspaces, and $P_g^\pm = U_g^\pm (U_g^\pm)^\top$.
\end{proposition}

\begin{proof}
By the Jordan--Wielandt theorem, the singular values of $(U_g^+)^\top U_g^-$ are exactly the principal cosines $\sigma_j = \cos \phi_j \in [0, 1]$ for $j \in \{1, \dots, k\}$. Thus:
\begin{equation}
\|(U_g^+)^\top U_g^-\|_F^2 = \sum_{j=1}^k \sigma_j^2 = \sum_{j=1}^k \cos^2 \phi_j.
\end{equation}
Using the cyclic commutativity of the trace:
\begin{equation}
\operatorname{tr}(P_g^+ P_g^-) = \operatorname{tr}\bigl(U_g^+ (U_g^+)^\top U_g^- (U_g^-)^\top\bigr) = \operatorname{tr}\bigl((U_g^+)^\top U_g^- (U_g^-)^\top U_g^+\bigr) = \|(U_g^+)^\top U_g^-\|_F^2.
\end{equation}
Dividing by $k$ yields the claim.
\end{proof}

\begin{proposition}[Expected Overlap Under Isotropic Null Calibration]
\label{prop:null-subspace-baseline}
Let $U_g^+$ and $U_g^-$ be independent, uniformly distributed random subspaces on the Grassmannian $\mathbf{Gr}(r_g^\pm, D)$. Then:
\begin{equation}
\mathbb{E}\bigl[\mathcal{O}_g^2\bigr] = \frac{\max(r_g^+, r_g^-)}{D}.
\label{eq:app-random-subspace-baseline}
\end{equation}
\end{proposition}

\begin{proof}
By rotational invariance of the Haar measure, $\mathbb{E}[P_g^-] = \frac{r_g^-}{D} I_D$. Conditioning on $P_g^+$ and using independence:
\begin{equation}
\mathbb{E}\bigl[\operatorname{tr}(P_g^+ P_g^-) \mid P_g^+\bigr] = \operatorname{tr}\left(P_g^+ \frac{r_g^-}{D} I_D\right) = \frac{r_g^-}{D} \operatorname{tr}(P_g^+) = \frac{r_g^+ r_g^-}{D}.
\end{equation}
Dividing by $k = \min(r_g^+, r_g^-)$ yields $\frac{r_g^+ r_g^-}{D \min(r_g^+, r_g^-)} = \frac{\max(r_g^+, r_g^-)}{D}$.
\end{proof}

\begin{remark}[Null Calibration Context]
Proposition~\ref{prop:null-subspace-baseline} demonstrates that when the ambient channel dimension is large ($D \gg r$), independent random subspaces yield an expected overlap near zero. The empirical overlap observed in Figure~1(a) ($\mathcal{O}_C \approx 1.0$) thus reflects mathematically profound shared structure between correct and wrong rollouts.
\end{remark}

\begin{proposition}[Projector Distance Equals Square Root of Symmetric Difference]
\label{prop:projector-symmetric-difference}
Let $C^{(t)}, C^{(\mathrm{pre})} \subseteq \{1, \dots, d\}$ be channel sets with diagonal projectors $P_C^{(t)} = \operatorname{diag}(\mathbf{1}_{C^{(t)}})$. The Frobenius distance between projectors satisfies:
\begin{equation}
\|P_C^{(t)} - P_C^{(\mathrm{pre})}\|_F = \sqrt{|C^{(t)} \mathbin{\triangle} C^{(\mathrm{pre})}|},
\label{eq:app-channel-symmetric-difference}
\end{equation}
where $\mathbin{\triangle}$ denotes symmetric set difference.
\end{proposition}

\begin{proof}
$P_C^{(t)} - P_C^{(\mathrm{pre})}$ is a diagonal matrix whose $j$-th entry is non-zero if and only if index $j$ belongs to one set but not the other:
\begin{equation}
(P_C^{(t)} - P_C^{(\mathrm{pre})})_{jj} = 
\begin{cases}
+1, & j \in C^{(t)} \setminus C^{(\mathrm{pre})}, \\
-1, & j \in C^{(\mathrm{pre})} \setminus C^{(t)}, \\
0, & \text{otherwise}.
\end{cases}
\end{equation}
Squaring the diagonal entries and summing across all $j \in \{1, \dots, d\}$ yields $|C^{(t)} \mathbin{\triangle} C^{(\mathrm{pre})}|$. Taking the square root proves Equation~(3) of the main text.
\end{proof}

\begin{proposition}[Exact Energy Decomposition of Baseline-Relative Drift]
\label{prop:relative-drift-decomposition}
Let the channels $\{1, \dots, d\}$ be partitioned into orthogonal coherence channels $C$ and reasoning channels $R$, with projectors $P_C^{(0)}$ and $P_R^{(0)} = I - P_C^{(0)}$. For any activation displacement $\Delta H = H^{(1)} - H^{(0)}$, define raw drift $F_g$, normalized percentage drift $A_g$, and baseline energy weight $w_g$:
\begin{equation}
F_g = \|\Delta H P_g^{(0)}\|_F, \qquad A_g = 100 \times \frac{F_g}{\|H^{(0)} P_g^{(0)}\|_F}, \qquad w_g = \frac{\|H^{(0)} P_g^{(0)}\|_F^2}{\|H^{(0)}\|_F^2}, \quad g \in \{C, R\}.
\end{equation}
Then the total normalized activation drift satisfies:
\begin{equation}
\frac{\|\Delta H\|_F^2}{\|H^{(0)}\|_F^2} = w_C \left(\frac{A_C}{100}\right)^2 + w_R \left(\frac{A_R}{100}\right)^2.
\label{eq:app-relative-drift-decomposition}
\end{equation}
\end{proposition}

\begin{proof}
Because $P_C^{(0)}$ and $P_R^{(0)}$ project onto orthogonal coordinate subsets, the squared Frobenius norms decompose additively:
\begin{align}
\|\Delta H\|_F^2 &= \|\Delta H P_C^{(0)}\|_F^2 + \|\Delta H P_R^{(0)}\|_F^2 = F_C^2 + F_R^2, \\
\|H^{(0)}\|_F^2 &= \|H^{(0)} P_C^{(0)}\|_F^2 + \|H^{(0)} P_R^{(0)}\|_F^2.
\end{align}
Dividing both sides of the displacement sum by $\|H^{(0)}\|_F^2$:
\begin{equation}
\frac{\|\Delta H\|_F^2}{\|H^{(0)}\|_F^2} 
= \frac{\|H^{(0)} P_C^{(0)}\|_F^2}{\|H^{(0)}\|_F^2} \left(\frac{F_C}{\|H^{(0)} P_C^{(0)}\|_F}\right)^2 
+ \frac{\|H^{(0)} P_R^{(0)}\|_F^2}{\|H^{(0)}\|_F^2} \left(\frac{F_R}{\|H^{(0)} P_R^{(0)}\|_F}\right)^2.
\end{equation}
Substituting $w_g$ and $A_g / 100$ establishes the decomposition, confirming the validity of Equation~(4).
\end{proof}

\end{document}